\documentclass[10pt,twocolumn,letterpaper]{article}

\usepackage{iccv}
\usepackage{times}
\usepackage{epsfig}
\usepackage{graphicx}
\usepackage{amsmath}
\usepackage{amsthm}
\usepackage{amssymb}
\usepackage{siunitx}

\newtheorem{lemma}{Lemma}[section]
\newtheorem{thm}{Theorem}[section]
\usepackage{algorithm}
\usepackage{algorithmicx}
\usepackage{algpseudocode}
\usepackage{tabto}

\newcommand{\subjectto}{\ensuremath{\mathrm{s.t.}\ }}
\newcommand{\mt}{\ensuremath{\mathsf{T}}}
\newcommand{\m}[1]{\ensuremath{\mathrm{\mathbf{#1}}}}

\newcommand{\s}[1]{\ensuremath{#1}}

\newcommand{\rbr}[1]{\ensuremath{\left(#1\right)}}

\newcommand{\rbrs}[1]{\ensuremath{(#1)}}

\newcommand{\cbrs}[1]{\ensuremath{\{#1\}}}

\newcommand{\mathsc}[1]{{\normalfont\textsc{#1}}}
\newcommand{\OPT}{\mathsc{opt}}
\newcommand{\opt}{\mathsc{opt}}
\newcommand{\LB}{\mathsc{LB}}
\newcommand{\UB}{\mathsc{UB}}

\algnewcommand{\TabComment}[2]{\tabto{#1} $\triangleright$ #2}

\iccvfinalcopy 

\def\iccvPaperID{1812} 

\ificcvfinal\fi
\begin{document}

\title{Learning for Active 3D Mapping}

\author{Karel Zimmermann, Tom{\'a}{\v s} Pet{\v r}{\'\i}{\v c}ek, Vojt{\v e}ch {\v S}alansk{\'y}, and Tom{\'a}{\v s} Svoboda\\
Czech Technical University in Prague, Faculty of Electrical Engineering\\
{\tt\small zimmerk@fel.cvut.cz}
}

\maketitle

\thispagestyle{empty}
\begin{abstract}
We propose an active 3D mapping method for depth sensors, which allow individual control of depth-measuring rays, such as the newly emerging solid-state lidars.
The method simultaneously (i) learns to reconstruct a dense 3D occupancy map from sparse depth measurements, and (ii) optimizes the reactive control of depth-measuring rays.
To make the first step towards the online control optimization, we propose a fast prioritized greedy algorithm, which needs to update its cost function in only a small fraction of possible rays.
The approximation ratio of the greedy algorithm is derived.
An experimental evaluation on the subset of the KITTI dataset demonstrates significant improvement in the 3D map accuracy when learning-to-reconstruct from sparse measurements is coupled with the optimization of depth measuring rays.
\end{abstract}

\section{Introduction}

In contrast to rotating lidars, the SSL uses an optical phased array as a transmitter of depth measuring light pulses. Since the built-in electronics can independently steer pulses of light by shifting its phase as it is projected through the array, the SSL can focus its attention on the parts of the scene important for the current task. Task-driven reactive control steering hundreds of thousands of rays per second using only an on-board computer is a challenging problem, which calls for highly efficient parallelizable algorithms. As a first step towards this goal, we propose an active mapping method for SSL-like sensors, which simultaneously (i) learns to \emph{reconstruct a dense 3D voxel-map} from sparse depth measurements and (ii) optimize the reactive \emph{control of depth-measuring rays}, see Figure~\ref{fig:diagram}.
The proposed method is evaluated on a subset of the KITTI dataset~\cite{Geiger-2013-IJRR}, where sparse SSL measurements are artificially synthesized from captured lidar scans,
and compared to a state-of-the-art 3D reconstruction approach~\cite{Choy-2016-ECCV}.

\begin{figure}[t]\label{fig:diagram}
	\begin{center}
		\includegraphics[width=1\linewidth]{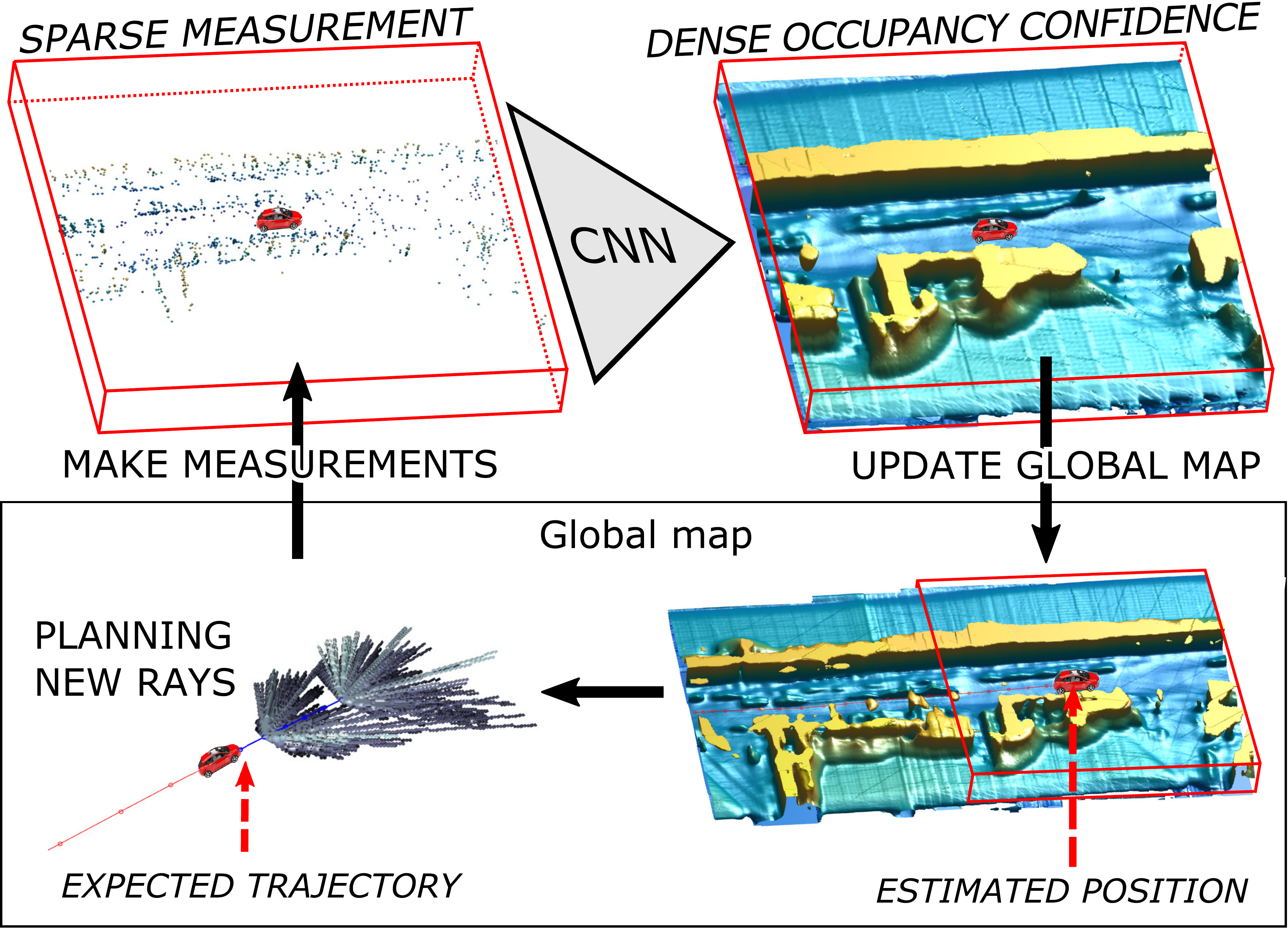}
	\end{center}
	\caption{\textbf{Active 3D mapping with Solid State Lidar.}
		Iteratively learned deep convolutional network reconstructs local dense occupancy map from sparse depth measurements.
		The local map is registered to a global occupancy map, which in turn serves as an input for the optimization of depth-measuring rays along the expected vehicle trajectory.
		The dense occupancy maps are visualized as isosurfaces. See \texttt{https://www.youtube.com/watch?v=KNex0zjeGYE}}
\end{figure}

The main contribution of this paper lies in proposing a computationally tractable approach for very high-dimensional active perception task, which couples learning of the 3D reconstruction with the optimization of depth-measuring rays. Unlike other approaches such as active object detection~\cite{Jayaraman-2016-ECCV} or segmentation~\cite{Mishra-2012-TPAMI}, SSL-like reactive control has significantly higher dimensionality of the state-action space, which makes a direct application of unsupervised reinforcement learning~\cite{Jayaraman-2016-ECCV} prohibitively expensive. 
Keeping the on-board reactive control in mind, we propose prioritized greedy optimization of depth measuring rays, which in contrast to a na{\" i}ve greedy algorithm re-evaluates only $1/500$ rays in each iteration. We derive the approximation ratio of the proposed algorithm.


\ifx\false
Active perception has been widely studied in many robotics applications ranging from exploration and active SLAM~\cite{MartinezCantin-2009-AR} to active object detection~\cite{Jayaraman-2016-ECCV} and segmentation~\cite{Mishra-2012-TPAMI}. In contrast to these applications, the active SSL-mapping has significantly higher dimensionality of the action space, which makes a direct application of known approaches such as unsupervised reinforcement learning~\cite{Jayaraman-2016-ECCV} prohibitively expensive. Unlike the active SLAM, we are mostly interested in a local 3D map, which allows for reactive control of the vehicle, rather than the global map which would require global motion rectifications such as loop closures. In such a short horizon, the motion estimated from odometry, IMU and GPS is sufficiently accurate, therefore localization is not an issue. In contrast to the discriminative voxel reconstruction approaches, the map being reconstructed has also significantly higher dimensionality.
Nevertheless we follow the main paradigm, which achieved state-of-the-art performance in the most of the active perception tasks: \emph{discriminative learning of the target perception task coupled with the active component}. 
\fi


The 3D mapping is handled by an iteratively learned convolution neural network (CNN), as CNNs proved their superior performance in~\cite{Choy-2016-ECCV,Wu-2015-CVPR}. 
The iterative learning procedure stems from the fact that both (i) the directions in which the depth should be measured and (ii) the weights of the 3D reconstruction network are unknown.
We initialize the learning procedure by selecting depth-measuring rays randomly to learn an initial 3D mapping network which estimates occupancy of each particular voxel.
Then, using this network, depth-measuring rays along the expected vehicle trajectory can be planned based on the expected reconstruction (in)accuracy in each voxel.
To reduce the training-planning discrepancy, the mapping network is re-learned on optimized sparse measurements and the whole process is iterated until validation error stops decreasing.

\section{Previous work}





High performance of image-based models is demonstrated in \cite{Su-2015-ICCV}, where a CNN pooling results from multiple rendered views outperforms commonly used 3D shape descriptors in object recognition task.
Qi \etal~\cite{Qi-2016-CVPR} compare several volumetric and multi-view network architectures and propose an anisotropic probing kernel to close the performance gap between the two approaches.
Our network architecture uses a similar design principle.

%
Choy \etal~\cite{Choy-2016-ECCV} proposed a unified approach for single and multi-view 3D object reconstruction which employs a recurrent neural architecture.
Despite providing competitive results in the object reconstruction domain, the architecture is not suitable for dealing with high-dimensional outputs due to its high memory requirements and would need significant modifications to train with full-resolution maps which we use.
We provide a comparison of this method to ours in Sec. \ref{sec:choy}, in a limited setting.


Model-fitting methods such as \cite{Shen-2012-TOG,Sung-2015-TOG,Rock-2015-CVPR} rely on a manually-annotated dataset of models and assume that objects can be decomposed into a predefined set of parts.
Besides that these methods are suited mostly for man-made objects of rigid structure, fitting of the models and their parts to the input points is computationally very expensive; e.g., minutes per input for \cite{Shen-2012-TOG,Sung-2015-TOG}.
Decomposition of the scene into plane primitives as in \cite{Monszpart-2015-TOG} does not scale well with scene size (quadratically due to candidate pairs) and could not most likely deal with the level of sparsity we encounter.

Geometrical and physical reasoning comprising stability of objects in the scene is used by Zheng \etal~\cite{Zheng-2013-CVPR} to improve object segmentation and 3D volumetric recovery.
Their assumption of objects being aligned with coordinate axes which seems unrealistic in practice.
Moreover, it is not clear how to incorporate learned shape priors for complex real-world objects which were shown to be beneficial for many tasks (e.g., in \cite{Nguyen-2016-CVPR}).
Firman \etal~\cite{Firman-2016-CVPR} use a structured-output regression forest to complete unobserved geometry of tabletop-sized objects.
A generative model proposed by Wu \etal~\cite{Wu-2015-CVPR}, termed Deep Belief Network, learns joint probability distribution $p(\m{x}, y)$ of complex 3D shapes $\m{x}$ across various object categories $y$.

End-to-end learning of stochastic motion control policies for active object and scene categorization is proposed by Jayaraman and Grauman~\cite{Jayaraman-2016-ECCV}.
Their CNN policy successively proposes views to capture with RGB camera to minimize categorization error.
The authors use a look-ahead error as an unsupervised regularizer on the classification objective.
%
Andreopoulos \etal~\cite{Andreopoulos-TRO-2011} solve the problem of an active search for an object in a 3D environment.
While they minimize the classification error of a single yet apriori unknown voxel containing the searched object, we minimize the expected reconstruction error of all voxels.
Also, their action space is significantly smaller than ours because they consider only local viewpoint changes at the next position while the SSL planning chooses from tens of thousands of rays over a longer horizon.

\newcommand{\ELL}{\mathcal{L}}



\section{Overview of the active 3D mapping}\label{sec:overview}

We assume that the vehicle follows a known path consisting of $L$ discrete positions and a depth measuring device (SSL) can capture at most $K$ rays at each position. The set of rays to be captured at position $l$ is denoted $J_l$.

We denote $\m{Y}$ the global ground-truth occupancy map, $\hat{\m{Y}}$ its estimate, and $\m{X}$ the map of the sparse measurements. All these map share common global reference frame corresponding to the first position in the path.
For each of these maps there are local counterparts $\m{y}_l, \hat{\m{y}}_l$, and $\m{x}_l$, respectively. 
Local maps corresponding to position $l$ all share a common reference frame which is aligned with the sensor and captures its local neighborhood.
The global ground-truth map $\m{Y}$ is used to synthesize sensor measurements $\m{x}_l$ and to generate local ground-truth maps $\m{y}_l$ for training. 

The active mapping pipeline,
consisting of a measure-reconstruct-plan loop,
is depicted in Fig.~\ref{fig:diagram} and detailed in Alg.~\ref{algo:pipeline}.
\begin{algorithm}
	\begin{algorithmic}[1]
		\item Initialize position $l \gets 0$ and select depth-measuring rays randomly.
		\item Measure depth in the directions selected for position $l$ and update global sparse measurements $\m{X}$ and dense reconstruction $\hat{\m{Y}}$ with these measurements.
		\item Obtain local measurements $\m{x}_l$ by interpolating $\m{X}$.
		\item Compute local occupancy confidence $\hat{\m{y}}_l = \m{h}_\theta(\m{x}_l)$ using the mapping network $\m{h}_\theta$.
		\item Update global occupancy confidence $\hat{\m{Y}} \gets \hat{\m{Y}} + \hat{\m{y}}_l$.
		\item Plan depth-measuring rays along the expected vehicle trajectory over horizon $L$ given reconstruction $\hat{\m{Y}}$.
		\item Repeat from line 2 for next position $l \gets l + 1$.
	\end{algorithmic}
	\caption{Active mapping}
	\label{algo:pipeline}
\end{algorithm}
Neglecting sensor noise, the set of depth-measuring rays obtained from the planning, the measurements $\m{x}_l$, and the resulting reconstruction $\hat{\m{Y}}$ can all be seen as a deterministic function of mapping parameters $\theta$ and $\m{Y}$.
If we assume that that ground-truth maps $\m{Y}$ come from a probability distribution,
both learning of $\theta$ and planning of the depth-measuring rays approximately minimize common objective
\begin{flalign}
\mathbb{E}_\m{Y}\cbrs{\ELL\rbrs{\m{Y}, \hat{\m{Y}}(\theta, \m{Y})}},\label{eq:common-obj}
\end{flalign}
where $\ELL(\m{Y}, \hat{\m{Y}}) = \sum_i w_i \log(1 + \exp(-Y_i \hat{Y}_i))$ is the weighted logistic loss, $Y_i \in \{-1, 1\}$ and $\hat{Y}_i \in \mathbb{R}$ denote the elements of $\m{Y}$ and $\hat{\m{Y}}$, respectively, corresponding to voxel $i$.
In learning, $w_i \ge 0$ are used to balance the two classes, \emph{empty} with $Y_i=-1$ and \emph{occupied} with $Y_i=1$, and to ignore the voxels with unknown occupancy.
We assume independence of measurements and use, for corresponding voxels $i$, additive updates of the occupancy confidence $\hat{Y}_i \gets \hat{Y}_i + h_i(\m{x}_l)$ with $h_i(\m{x}_l) \approx \log(\mathrm{Pr}(Y_i = 1 | \m{x}_l) / \mathrm{Pr}(Y_i = -1 | \m{x}_l))$.
$\mathrm{Pr}(Y_i = 1 | \m{x}_l)$ denotes the conditional probability of voxel $i$ being occupied given measurements $\m{x}_l$ and $\sigma(\hat{Y}_i) = 1 / (1 + e^{-\hat{Y}_i})$ is its current estimate.

\ifx\false
We assume that the voxel $i$ is visible in ray $j$ which intersects sequence of voxels $R$. If all $i$th preceding voxels $R^-(i)$ are not occupied and the voxel itself or at least one of the voxels which follow $R^+(i)$ are occupied.
Consequently, we estimate probability $p_{ij}$ of voxel $i$ \emph{not} being visible in $j$ as 
$p_{ij} = 1 - \prod_{u\in R^-(i)}(1-q_{u})\left(1 - \prod_{u\in R^+(i)}(1-q_{u})\right).$ 
If ray $j$ does not intersect the voxel $i$, then $p_{ij} = 1$.
\fi


\ifx\false
Dense map reconstruction is tackled by the mapping network $\m{h}_\theta$. Result of the reconstruction $\hat{\m{y}}_l = \m{h}_\theta(\m{x}_l)$ at position $l$ depends on the network parameters $\theta$ and sparse depth measurements $\m{x}_l$ determined by previously captured rays $J_1\dots J_l$. 
We define learning of $\theta$ as the minimization of the cross entropy loss $\sum_l\mathcal{H}\{\hat{\m{y}}_l, \m{y}_l\}$ between local reconstructions $\hat{\m{y}}_l$ from captured sparse measurement and the local ground truth maps $\m{y}_l$. 
The minimization is tackled by the Stochastic Gradient Descent detailed in Section~\ref{sec:mapping}. We denote one step of the SGD as follows:
$$
\theta^t = \texttt{SGD}(\theta^{t-1},J(\theta)) 
$$ 
where $J(\theta)$ denotes concatenation of all rays at all positions (in all training maps).

The planning at position $l$ of following rays $J_{l+1}\dots J_{l+L}$ for the horizon $L$ is defined as the minimization of the expected cross-entropy loss $E_{\m{Y}\sim \hat{\m{Y}}_l}\{\mathcal{H}(\hat{\m{Y}}_{l+L}, \m{Y})\}$ subject to the limited ray budget $K$. Planning is tackled by proposed Prioritized Greedy algorithm detailed in Section~\ref{sec:planning}. We denote Prioritized Greedy planning of all rays $J$ at all positions (and maps) as follows:
$$
J^t = \texttt{PG}(\theta, J^{t-1})
$$ 
\fi

\newenvironment{nospaceflalign}
{\setlength{\abovedisplayskip}{4pt}\setlength{\belowdisplayskip}{4pt}%
	\csname flalign\endcsname}
{\csname endflalign\endcsname\ignorespacesafterend}

\newenvironment{nospacetopflalign}
{\setlength{\abovedisplayskip}{2pt}\setlength{\belowdisplayskip}{4pt}%
	\csname flalign\endcsname}
{\csname endflalign\endcsname\ignorespacesafterend}

\newenvironment{nospacebelowflalign}
{\setlength{\abovedisplayskip}{4pt}\setlength{\belowdisplayskip}{2pt}%
	\csname flalign\endcsname}
{\csname endflalign\endcsname\ignorespacesafterend}

\section{Learning of 3D mapping network}\label{sec:mapping}


The learning is defined as approximate minimization of Equation~\ref{eq:common-obj}.
Since (i) the result of planning $\m{x}_l\rbr{\theta,\m{Y}}$ is not differentiable with respect to $\theta$ and
(ii) we want to reduce variability of training data\footnote{We introduce a canonical frame by using the local maps instead of the global ones, which helps the mapping network to capture local regularities.},
we locally approximate the criterion around a point $\theta^0$ as
\begin{nospacebelowflalign}
&\mathbb{E}_\m{Y}\{\sum_{l}\ELL(\m{y}_l,\m{h}_\theta(\m{x}_l(\theta^0,\m{Y})))\}
\end{nospacebelowflalign}
by fixing the result of planning in $\m{x}_l(\theta^0,\m{Y})$.
We also introduce a canonical frame by using the local maps instead of the global ones, which helps the mapping network to capture local regularities.
The learning then becomes the following iterative optimization 
\begin{nospacebelowflalign}
\theta^t=\arg\min_\theta \mathbb{E}_\m{Y}\cbrs{\sum_l\ELL\rbrs{
    \m{y}_l,\m{h}_\theta(\m{x}_l(\theta^{t-1},\m{Y}))
}},
\label{eq:opt-approx-common-obj}
\end{nospacebelowflalign}
\ifx\false
Note, that in order to achieve (i) local optimality of the criterion and (ii) statistical consistency of the learning process (i.e. that the training distribution of sparse measurements $\m{x}_l$ corresponds to the one obtained by planning), one would have to find a fixed point of Equation~\ref{eq:opt-approx-common-obj}. Since there are no guarantees that any fixed point exists, we instead iterate the minimization until validation error is decreasing.
\fi
where minimization in each iteration is tackled by Stochastic Gradient Descent. Learning is summarized in Alg.~\ref{algo:learning}.


%
%

\begin{algorithm}
	\begin{algorithmic}[1]
        \item Initialize $t\gets 0$ and obtain dataset $\s{D}_0 = \{(\m{x}_l, \m{y}_l)\}_l$ by running the pipeline with the rays being selected randomly, instead of using the planner. 
        \item Train the mapping network on $\s{D}_t$ to obtain $\m{h}_{\theta^t}$ with parameters $\theta^t$.
        \item Obtain $\s{D}_{t+1} = \{(\m{x}_l(\theta^t,\m{Y}), \m{y}_l)\}_l$ by running Alg.~\ref{algo:pipeline} and using $\m{h}_{\theta^t}$ for mapping.
        \item Set $t \gets t + 1$ and repeat from line 2 until validation error stops decreasing.
	\end{algorithmic}
	\caption{Learning of active mapping}
	\label{algo:learning}
\end{algorithm}

Note, that in order to achieve (i) local optimality of the criterion and (ii) statistical consistency of the learning process (i.e., that the training distribution of sparse measurements $\m{x}_l$ corresponds to the one obtained by planning), one would have to find a fixed point of Equation~\ref{eq:opt-approx-common-obj}. Since there are no guarantees that any fixed point exists, we instead iterate the minimization until validation error is decreasing.

\ifx\false
Initial learning parameters $\theta^0$ are obtained by training the mapping network on dataset $\s{D}_0 = \{(\m{x}_l, \m{y}_l)\}_l$, where sparse measurements $\m{x}_l$ are randomly generated from available depth-measuring rays. 
This network provides reconstructions of the global map for all positions $l$. Planning on these maps determines sparse depth measurements $\m{x}_l(\theta^0,\m{Y})$.  
In the following iterations $t\geq1$, the mapping network  $\m{h}_{\theta^t}$ is always trained on dataset $\s{D}_t = \{(\m{x}_l(\theta^0,\m{Y}), \m{y}_l)\}_l$. 




\fi

\ifx\false
We seek a CNN which reconstructs local 3D occupancy map $\m{y}_l$ from sparse measurements $\m{x}_l$ selected by a planner minimizing the expected reconstruction error.
Since sparse depth measurements are planned in the global map estimated by CNN
and, on the other side, the CNN estimates the global map from the sparse depth measurements estimated by the planner, the planning and learning are mutually interconnected.
\fi


\begin{figure}[t]
	\begin{center}
		\includegraphics[width=0.92\linewidth]{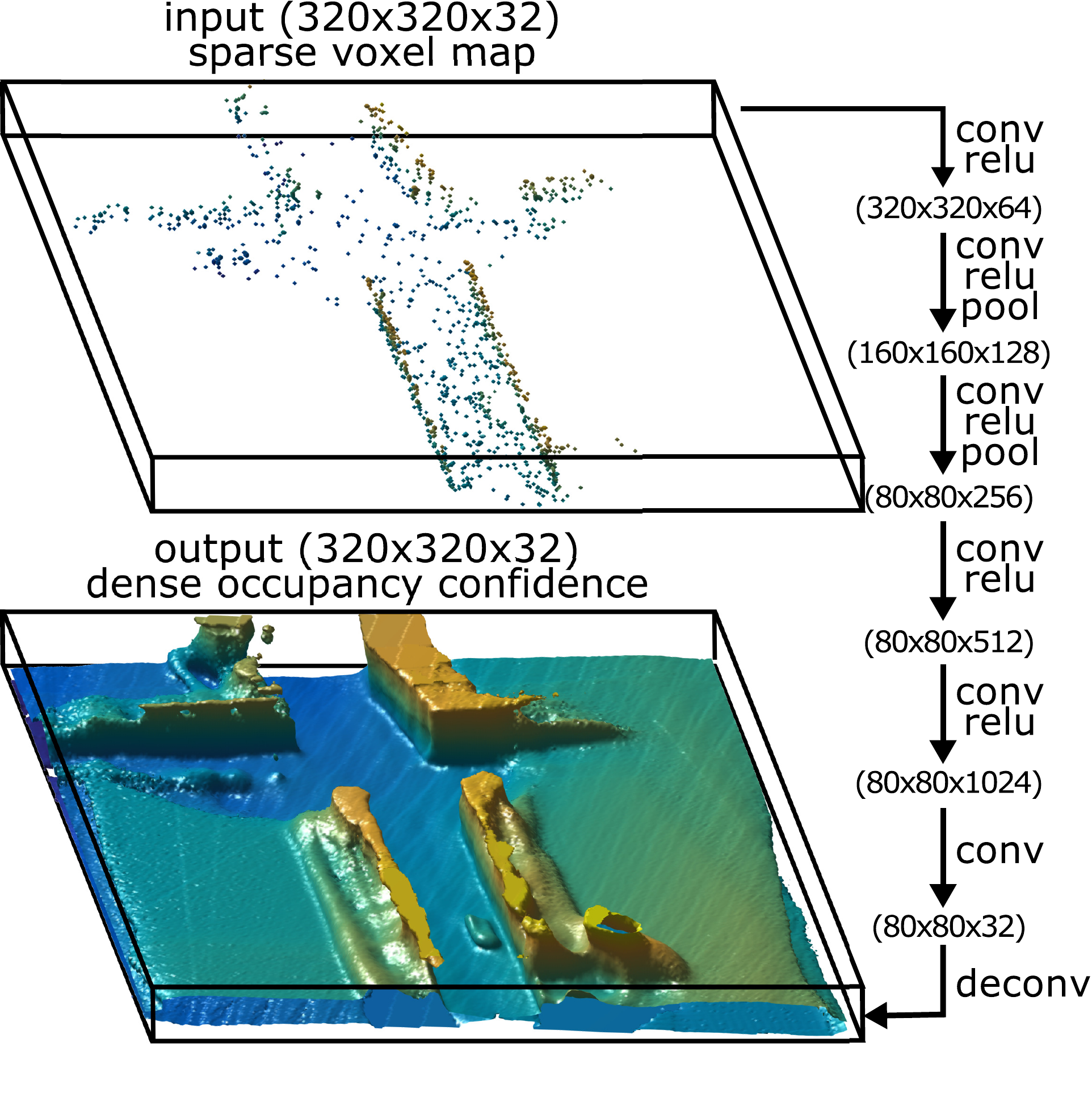}
	\end{center}
	\caption{Architecture of the mapping network.
		\textbf{Top:} An example input with sparse measurements, showing only the occupied voxels.
		\textbf{Bottom:} The corresponding reconstructed dense occupancy confidence after thresholding.
		\textbf{Right:} Schema of the network architecture, composed from the convolutional layers (denoted \emph{conv}), linear rectifier units (\emph{relu}), pooling layers (\emph{pool}), and upsampling layers (\emph{deconv}).}
	\label{fig:net}
\end{figure}

The mapping network consists of 6 convolutional layers with $5 \times 5$ kernels followed by linear rectifier units (element-wise $\max\{x,0\}$) and, in 2 cases, by max pooling layers with $2 \times 2$ kernels and stride $2$, see Fig.~\ref{fig:net}.
In the end, there is an fourfold upsampling layer so that the output has same size as input.
The network was implemented in {\emph{MatConvNet}}~\cite{Vedaldi-2015-ICMM}.


\section{Planning of depth measuring rays}\label{sec:planning}

\ifx\false
We assume that the vehicle follows a known path consisting from $L$ discrete positions, the SSL can capture at most $K$ rays at each position. We search for the subset of rays along this path, which decreases the logistic loss the most. Since it is not clear how to quantify the impact of measuring a subset of voxels on the CNN logistic loss, we simplify the problem and assume that measuring a voxel $i$ decreases the overall logistic loss only by its current expected loss $\epsilon_i\ELL(Y_i, Y^*_i)$. Measuring the depth in the set of rays $J$ yields for each voxel $i$, the probability $\m{p}_i(J)$ to be measured. 

 
Each voxel $i$ has probability $p_{ij}$ \textbf{not} to be covered by ray $j$. 
Given global confidence map $Y$ we estimate the probability $q_i = 1 / \left(1 + e^{-Y_i}\right)$ of voxel $i$ being occupied.
We assume that the voxel $i$ is visible in ray $j$ which intersects sequence of voxels $R$. If all $i$th preceding voxels $R^-(i)$ are not occupied and the voxel itself or at least one of the voxels which follow $R^+(i)$ are occupied.
Consequently, we estimate probability $p_{ij}$ of voxel $i$ \emph{not} being visible in $j$ as 
$p_{ij} = 1 - \prod_{u\in R^-(i)}(1-q_{u})\left(1 - \prod_{u\in R^+(i)}(1-q_{u})\right).$ 
If ray $j$ does not intersect the voxel $i$, then $p_{ij} = 1$.

We assume that when a voxel is covered by more than one ray, their contributions are independent, therefore probability of voxel $i$ being not covered by any ray from the set of rays $J$ is equal to $p_i(J) = \prod_{j\in J} p_{ij}$. 
We define the planning as the search for the sequence $J=\{J_1, \dots J_L\}$ of subsets of depth-measuring rays $J_1, \dots J_L$, which minimize the expected sum of reconstruction errors $E_{I\sim p(J)}\{\sum_{i\in I} \epsilon_i\}\approx\m{\epsilon}^\mt\m{p}(J) $, such that $|J_1|\leq K, \dots |J_L|\leq K$. 
\fi

Planning at position $l$ searches for a set of rays $J$, which approximately minimizes the expected logistic loss $\ELL(\m{Y},\m{h}_{\theta^t}(\m{x}_{l+L}))$ between ground truth map $\m{Y}$ and reconstruction obtained from sparse measurements $\m{x}_{l+L}$ at the horizon $L$.
The result of planning is the set of rays $J$, 
which will provide measurements for a sparse set of voxels. This set of voxels is referred to as \emph{covered} by $J$ and denoted as $C(J)$.
While the mapping network is trained \emph{offline} on the ground-truth maps, the planning have to search the subset of rays \emph{online} without any explicit knowledge of the ground-truth occupancy $\m{Y}$. 
Since it is not clear how to directly quantify the impact of measuring a subset of voxels on the reconstruction $\m{h}_{\theta^t}(\m{x}_{l+L})$, we introduce simplified reconstruction model $\hat{\m{h}}(J,\hat{\m{Y}})$, which predicts the loss based on currently available map $\hat{\m{Y}}$. This model conservatively assumes that the reconstruction in covered voxels $i\in C(J)$ is correct (i.e. $\ELL\big(Y_{i},\hat{h}_i(J,\hat{\m{Y}})\big) = 0$) and reconstruction of not covered voxels $i\notin C(J)$ does not change (i.e. $\ELL\big(Y_{i},\hat{h}_i(J,\hat{\m{Y}})\big) = \ELL(Y_{i},\hat{Y}_{i})$). Given this reconstruction model, the expected loss simplifies to:
\begin{nospaceflalign}
&\mkern-5mu\sum_i\ELL\big(Y_{i},\hat{h}_i(J,\hat{\m{Y}})\big)= 
\sum_{i\notin C(J)}\ELL(Y_i,\hat{Y}_i)
\end{nospaceflalign}
Since the ground-truth occupancy of voxels is apriori unknown, neither the voxel-wise loss nor the coverage are known. We model the expected loss  
in voxel $i$ as 
\begin{flalign}
\ELL(Y_i,\hat{Y}_i) \approx
\mathbb{E}_{Y_{i}\sim \textrm{B}(\sigma(\hat{Y}_{i}))}\ELL(Y_{i},\hat{Y}_{i}) = \mathcal{H}(\textrm{B}(\sigma(\hat{Y}_i))) = \epsilon_i,
\end{flalign}
where $\mathcal{H}(\textrm{B}(p))$ is the entropy of the Bernoulli distribution with parameter $p$, denoting the probability of outcome $1$ from the possible outcomes $\{-1, 1\}$.
The vector of concatenated losses $\epsilon_i$ is denoted $\m{\boldsymbol\epsilon}$.

\ifx\false
Planning at position $l$ approximately minimizes the expected logistic loss $\ELL\big(\m{Y},\m{h}_\theta(\m{x}_{l+L}(\theta^0,\m{Y}))\big)$ at the horizon $L$.
The result of planning is the set of rays $J$, 
which will provide measurements for a sparse set of voxels. This set of voxels is referred to as \emph{covered} by $J$ and denoted as $C(J)$.
While the mapping network is trained \emph{offline} on the ground-truth maps, the planning have to search the subset of rays \emph{online} without any explicit knowledge of the ground-truth occupancy $\m{Y}$. 
Since it is not clear how to directly quantify the impact of measuring a subset of voxels on the reconstruction $\m{h}_\theta(\m{x}_{l+L}(\theta^0,\m{Y}))$, we introduce simplified reconstruction model $\hat{\m{h}}(J,\hat{\m{Y}})$. This model conservatively assumes that the reconstruction in covered voxels $i\in C(J)$ is correct (i.e. $\ELL\big(\m{y}_{l+L,i},\hat{\m{h}}_i(J,\hat{\m{Y}})\big) = 0$) and reconstruction of not covered voxels $i\notin C(J)$ does not change. Given this reconstruction model, the expected loss simplifies to:

\begin{flalign}
&\mkern-5mu\sum_i\ELL\big(\m{y}_{l+L,i},\hat{\m{h}}_i(J,\hat{\m{Y}})\big)= 
\sum_{i\notin C(J)}\ELL(\m{y}_{l+L,i},\hat{\m{y}}_{l+L,i})
\end{flalign}
Since the ground-truth occupancy of voxels is apriori unknown, neither the voxel-wise loss nor the coverage are known. We model the expected loss  
in voxel $i$ as 
\begin{flalign}
\mkern-15mu\epsilon_i = \mathbb{E}_{\m{y}_{l+L,i}\sim \textrm{Bern}(\sigma(\hat{\m{y}}_{l,i})}\ELL(\m{y}_{l+L,i},\hat{\m{y}}_{l+L,i})=\nonumber \\ =\mathcal{H}(\textrm{Bern}(\sigma(\hat{\m{y}}_{l,i})))
\end{flalign}
Vector of concatenated voxel-wise losses $\epsilon_i$ is denoted as $\boldsymbol\epsilon$.
\fi

The length of particular rays is also unknown, therefore coverage $C(J)$ of voxels by particular rays cannot be determined uniquely.
Consequently, we introduce probability $p_{ij}$ that voxel $i$ will not be covered by ray $j\in J$.
This probability is estimated from currently available map $\hat{\m{Y}}$ as the product of
(i) the probability that the voxels on ray $j$ which lie between voxel $i$ and the sensor are unoccupied and
(ii) the probability that at least one of the following voxels or the voxel $i$ itself are occupied.
If ray $j$ does not intersect voxel $i$, then $p_{ij} = 1$.
The vector of probabilities $p_{ij}$ for ray $j$ is denoted $\m{p}_j$.
Assuming that rays $J$ are independent measurements, the expected loss is modeled as $\m{\boldsymbol\epsilon}^\mt\prod_{j\in J} \m{p}_j$.
\ifx\false
Each ray $j$ intersects a sequence of voxels $R_j$ ordered by the distance from the sensor.
We assume that voxel $i$ is visible in ray $j$ if all voxels $R_j^-(i)$ which precede voxel $i$ are empty and the voxel itself or at least one of the voxels which follow $R_j^+(i)$ are occupied.
Consequently, we estimate probability $p_{ij}$ of voxel $i$ \emph{not} being covered by ray $j$ as $p_{ij} = 1 - \prod_{u\in R^-(i)}(1-q_{u}(\hat{\m{Y}}))\left(1 - \prod_{u\in R^+(i)}(1-q_{u}(\hat{\m{Y}}))\right).$

\fi

\ifx\false
\begin{flalign}
&\mkern-16mu\mathbb{E}_{\m{Y}\sim \hat{\m{Y}}_l}\{\mathcal{H}(\m{Y},\hat{\m{Y}}_{l+L})\} = \mathbb{E}_{\m{Y}\sim \hat{\m{Y}}_l}\{\mathcal{H}(\m{Y}, \hat{\m{Y}}_{l})\}\m{p}(J,\hat{\m{Y}}_l)
\end{flalign}
where the expected voxel-wise loss at the current position $l$ reduces to the entropy of the current map $\hat{\m{Y}}_{l}$:
\begin{flalign}
& \mathbb{E}_{\m{Y}\sim \hat{\m{Y}}_l}\{\mathcal{H}(\hat{\m{Y}}_{l}, \m{Y})\} = \mathcal{H}(1,\hat{\m{Y}}_{l})\hat{\m{Y}}_{l} + \mathcal{H}(0,\hat{\m{Y}}_{l})(1-\hat{\m{Y}}_{l})= \nonumber\\
&\mkern150mu=\mathcal{H}(\hat{\m{Y}}_{l}),
\end{flalign}
and $\m{p}(J,\hat{\m{Y}}_l)$ is the voxel-wise probability that particular voxels will \emph{not} be covered by the set of rays $J$. Note, that since the ground-truth occupancy of voxels is apriori unknown, the length of particular rays is also unknown and the coverage of voxels by particular rays cannot be determined uniquely, therefore probability $\m{p}(J,\hat{\m{Y}}_l)$ is estimated instead. Assuming that rays are independent measurements, the probability is 
$
\m{p}(J,\hat{\m{Y}}_l) = \prod_{j\in J} \m{p}_j,
$
where $\m{p}_j$ is the probability that voxels are not covered by ray $j\in J$ and $\prod_{j\in J} \m{p}_j$ is element-wise multiplication of vectors $\m{p}_j$.

Without loss of generality, we assume that planning is performed for a fixed position $l=0$, therefore index $l$ is dropped in the following text for simplicity. We assume that the current map $\hat{\m{Y}}$ determines vectorized voxel-wise probabilities $\m{p}_j$ and vectorized voxel-wise losses $\boldsymbol\epsilon = \mathcal{H}(\hat{\m{Y}})$, detailed description is available in section~\ref{sec:appendix-coverage_probability}. 
\fi

The planning searches for the set $J= J_1\cup\dots\cup J_L$ of subsets $J_1\dots J_L$ of depth-measuring rays for the following $L$ positions, which minimize the expected loss, subject to budget constraints $|J_1|\leq K, \dots |J_L|\leq K$
\begin{nospacebelowflalign}
\mkern-7mu J^* = \arg\min_J\;\m{\boldsymbol\epsilon}^\mt\prod_{j\in J} \m{p}_j,\, \subjectto |J_1|\leq K, \dots |J_L|\leq K, \label{eq:planning-problem}
\end{nospacebelowflalign}
%
where $|J_l|$ denotes cardinality of the set $J_l$.
 
 
This is a non-convex combinatorial problem\footnote{In our experiments, the number of possible combinations is greater then $10^{2000}$.} which needs to be solved online repeatedly for millions of potential rays. We tried several convex approximations, however the high-dimensional optimization has been extremely time consuming and the improvement with respect to the significantly faster greedy algorithm was negligible. As a consequence of that, we have decided to use the greedy algorithm. We first introduce its simplified version (Alg.~\ref{algo:greedy0}) and derive its properties, the significantly faster prioritized greedy algorithm (Alg.~\ref{algo:greedy}) is explained later. 

We denote the list of available rays at position $l$ as $V_l$. At the beginning, the list of all available rays is initialized as follows $V = V_1 \cup \dots \cup V_L$. Alg.~\ref{algo:greedy0} successively builds the set of selected rays $J$. In each iteration the best ray $j^*$ is selected, added into $J$ and removed from $V$. The position from which the ray $j^*$ is chosen is denoted $l^*$. If the budget $K$ of $l^*$ is reached, all rays from $V_{l^*}$ are removed from $V$. 

In order to avoid multiplication of all selected  rays at each iteration, we introduce the vector $\m{b}$, which keeps voxel loss. Vector $\m{b}$ is initialized  as $\m{b}=\m{\boldsymbol\epsilon}$ and whenever ray $j$ is selected, voxel losses are updated as follows $\m{b} = \m{b}\odot\m{p}_j$, where $\odot$ denotes element-wise multiplication. 


%
%
%
%
%
%

\begin{algorithm}[h]
    \small{
\begin{algorithmic}[1]
\Require Set of available rays $V$ and budget $K$
\State $J \gets \emptyset$ \TabComment{4.5cm}{Initialization}
\State $\m{b} \gets \m{\boldsymbol\epsilon}$
\While{$\neg(V=\emptyset)$}
    \State $j^* \gets \arg\min_{j\in V} \m{b}^\mt\m{p}_j$ \TabComment{4.5cm}{Add the best ray}
    \State $J \gets J\cup j^*$
    \State $\m{b} \gets \m{b}\odot\m{p}_j$ \TabComment{4.5cm}{Update voxel costs}
    \State $V \gets V\setminus j^*$ \TabComment{4.5cm}{Remove $j^*$ from $V$}
    \If{$|J_{l^*}|=K$} 
        \State $V \gets V\setminus V_{l^*}$ \TabComment{4.5cm}{Close position}
    \EndIf
\EndWhile
\State \Return Set of selected rays $J$
\end{algorithmic}
\caption{Greedy planning}
\label{algo:greedy0}
}
\end{algorithm}

The rest of this section is organized as follows: Section~\ref{sec:approximation-ratio} shows the upper bound for the approximation ratio of the greedy algorithm. Section~\ref{sec:faster-greedy} introduces the prioritized greedy algorithm, which in each iteration needs to re-evaluate the cost function $\m{b}^\mt\m{p}_j$ only for a small fraction of rays. 

\subsection{Approximation ratio of the greedy algorithm}\label{sec:approximation-ratio}

\ifx\false
We start by couple of lemmas, which are necessary for proving the main theorem. 

\begin{lemma}\label{lemma1}
	The following upper bound holds:
	\begin{eqnarray}
	 & \max_{\m{x},\m{y}} \sum_i x_iy_i \;\;\;\;\;\;\; & \leq \min\{A,B\} \\
	\subjectto & \sum_i x_i \leq A,\; 0\leq x_i \leq 1, & \nonumber \\
	           & \sum_i y_i \leq B,\; 0\leq y_i \leq 1. & \nonumber
	\end{eqnarray}
\end{lemma}
\begin{proof}
	Without loss of generality we assume that $B\leq A$. Then $\min\{A,B\}=B\geq \sum_i y_i \geq \sum_i x_i y_i$, because multiplying positive numbers $y_i$ by positive and smaller than one numbers $x_i$ can only decrease their values.
\end{proof}

\begin{lemma}\label{lemma2}
    For any $0\leq p \leq 1$:
	\begin{eqnarray}
    \label{lemma2-max}
	\max_{\m{x}} &\sum_{i=1}^k x_i& = k-1+p. \\
	\subjectto &\prod_{i=1}^k x_i = p& \label{lemma2-eq} \\
	&0\leq x_i \leq 1&\nonumber
	\end{eqnarray}
\end{lemma}
\begin{proof}
    Let us consider two cases separately.
    First, for $p = 0$, a single element $x_i = 0$ is sufficient to satisfy Eq.~(\ref{lemma2-eq}).
    The criterion can then be maximized element-wise by setting $(k-1)$ elements $x_{j\neq i} = 1$.
    Second, for $p > 0$, $x_i > 0$ must hold for every $i$ to satisfy Eq.~(\ref{lemma2-eq}).
    Now suppose that the maximizer contains at least two elements $x_i, x_j < 1$.
    It follows that
    \begin{eqnarray}
    (1 - x_i) x_j &<& 1 - x_i, \nonumber \\
    x_i + x_j &<& x_i x_j + 1.\label{lemma2-ineq}
    \end{eqnarray}
    Using $x_i' = x_i x_j$ and $x_j' = 1$ instead of $x_i$ and $x_j$, respectively, gives us another feasible point, since $x_i' x_j' = (x_i x_j) 1 = x_i x_j = p / \prod_{n \neq i,j} x_n$.
    By adding $\sum_{n \neq i,j} x_n$ to both sides of Eq.~(\ref{lemma2-ineq}) it can be seen, however, that a higher value is attained at this point.
    This contradicts the proposition that the maximizer contains at least two elements $x_i, x_j < 1$; the maximizer must contain at most one $x_i < 1$ with the rest of elements $x_{j\neq i} = 1$.
    From Eq.~(\ref{lemma2-eq}) it follows that $x_i = p$ and Eq.~(\ref{lemma2-max}) holds.
\end{proof}
\fi 

We define the approximation ratio of a minimization algorithm to be $\rho=\frac{f}{\OPT}$, where $f$ is the cost function achieved by the algorithm and $\OPT$ is the optimal value of the cost function. Given $\rho$, we know that the algorithm provides solution whose value is at most $\rho\ \OPT$. In this section we derive the upper bound of the approximation ratio $\UB(\rho)$ of Algorithm~\ref{algo:greedy0}. Figure~\ref{fig:approximation_ratio} shows values of $\UB(\rho)$ for different number of positions $L$.

The greedy algorithm successively selects rays that reduce the cost function the most. To show how cost function differs from $\OPT$, an upper bound on the cost function need to be derived. Let us suppose that in the beginning of an arbitrary iteration we have voxel losses given by vector $\m{b}$, the following lemma states that for arbitrary voxel $i$, there always exists a ray $j$, that reduces the cost function to $\sum_i b_i(1-\frac{1}{K}) + \frac{\OPT}{K}$, where $\OPT = \m{1}^\mt\prod_{j=1}^K \m{p}_{j} =\m{1}^\mt \m{p}^{\OPT}$ 
%
%
is the unknown optimum value of the cost function which is achievable by $K$ rays $\m{p}_1\dots \m{p}_K$.

\begin{lemma}\label{lemma3}
	If for some rays $\prod_{j=1}^K p_{ij}=p_i^{\opt}$ then 
	\begin{nospaceflalign}
	\forall_{\m{0}\leq \m{b} \leq \m{1}} \exists_j \sum_{i=1}^V p_{ij}b_i \leq \sum_{i=1}^V b_i(1-\frac{1}{K}) + \frac{\OPT}{K}
	\end{nospaceflalign}
\end{lemma}
\noindent\emph{Proof:}
	We know that there is optimal solution consisting from $K$ rays.
	Without loss of generality we assume that $\prod_{j=1}^K p_{ij}=p_i^{\opt}$ holds for first $K$ rays, then 
	\begin{nospaceflalign}
	\forall_i \sum_{j=1}^{K} p_{ij} \leq K-1+p^{\opt}_i.
	\end{nospaceflalign}
	This holds for an arbitrary positive scaling factor $b_i$, therefore
	\begin{nospacetopflalign}
	\forall_i \sum_{j=1}^{K} p_{ij}b_i \leq (K-1+p^{\opt}_i)b_i. 
	\end{nospacetopflalign}
	We sum up inequalities over all voxels $i$ 
	\begin{equation}
	\sum_{i=1}^V\sum_{j=1}^{K} p_{ij}b_i \leq \sum_{i=1}^V(K-1+p^{\opt}_i)b_i.
	\end{equation}
	We switch sums in the left hand side of the inequality to obtain addition of $K$ terms as follows  
	\begin{equation}
	\sum_{i=1}^V p_{i1}b_i + \dots + \sum_{i=1}^V p_{iK}b_i \leq \sum_{i=1}^V(K-1+p^{\opt}_i)b_i
	\end{equation}
	Hence, we know that at least one of these $K$ terms has to be smaller than or equal to $\frac{1}{K}$ of the right hand side
	\begin{eqnarray}
	\hspace{-0.4cm}\exists_j \sum_{i=1}^V p_{ij}b_i &\leq& \frac{1}{K}\sum_{i=1}^V(K-1+p^{\opt}_i)b_i= \nonumber \\
									&=&    \sum_{i=1}^Vb_i (1-\frac{1}{K})+\frac{1}{K}\sum_{i=1}^Vp^{\opt}_ib_i\leq \nonumber \\
									&\leq & \sum_{i=1}^Vb_i (1-\frac{1}{K})+\sum_{i=1}^V\frac{p^{\opt}_i}{K}= \\
									&=& \sum_{i=1}^Vb_i (1-\frac{1}{K})+\frac{\OPT}{K}  \nonumber\hspace{1.3cm}\square
	\end{eqnarray}
Especially, if there is only one position, all optimal $K$ rays $\m{p}_1\dots \m{p}_K$ are either already selected or still available. This assumption allows to derive the following upper bound on the cost function of the greedy algorithm $f^K$ after $K$ iterations for $L=1$.
\begin{thm}\label{theorem1}
	Upper bound $\mathit{\UB}(f^K)\geq f^K$ of the greedy algorithm after $K$ iterations is 
	\begin{nospaceflalign}
	\mathit{\UB}(f^K) = E \frac{1}{e} + \OPT\left(1-\frac{1}{e}\right), 
	\end{nospaceflalign} 
	 where $E = \sum_{i=1}^V \epsilon_i$ and $e$ is Euler number. 
\end{thm}
\noindent\emph{Proof:} We prove the upper bound by complete induction. In the beginning no ray is selected, per-voxel loss is $b_i^0 = \epsilon_i$ and the value of the cost function $f^0 = \sum_{i=1}^V b_i^0 = E$. Using Lemma~\ref{lemma3}, we know that there exists ray $j$ such that $\sum_{i=1}^V p_{ij}b_i^0 \leq \sum_{i=1}^V b_i^0(1-\frac{1}{k}) + \frac{\OPT}{K}$, therefore we know that
	\begin{nospacetopflalign}
	&f^1 =\sum_{i=1}^V p_{ij}b_i^0 \leq \sum_{i=1}^V b_i^0\left(1-\frac{1}{K}\right) + \frac{\OPT}{K}= \nonumber \\
	&	  =E \left(1-\frac{1}{K}\right) + \frac{\OPT}{K}. 
	\end{nospacetopflalign}
	Greedy algorithm continues by updating the per-voxel loss $b_i^1 = b_i^0p_{ij}$. In the second iteration there are two possible cases: (i) we have either used the optimal ray in the first iteration, then the situation is better and we know there is $(K-1)$ rays which achieves optimum, or (ii)  we have not selected the optimal ray in the first iteration, therefore we have still $K$ rays which achieves the optimum. Since the cost function reduction in the latter case gives the upper bound on the cost function reduction in the former one, we assume that there is still $k$ optimal rays available, therefore there exists ray $j$ such that
	\begin{eqnarray}
	\hspace{-3cm}f^2 &=&\sum_{i=1}^V p_{ij}b_i^1\leq \sum_{i=1}^V b_i^1\left(1-\frac{1}{k}\right) + \frac{\OPT}{K}\leq \nonumber \\
	&\leq& E \left(1-\frac{1}{K}\right)^2 + \frac{\OPT}{K}\left(\left(1-\frac{1}{K}\right)+1\right). 
	\end{eqnarray}
	We assume that the following holds
	\begin{eqnarray}
	\hspace{-0.3cm}
	f^{t-1} &\leq& E \left(1-\frac{1}{K}\right)^{t-1} + \frac{\OPT}{K}\sum_{u=0}^{t-2}\left(1-\frac{1}{K}\right)^u. \label{eq:assumption}
	\end{eqnarray}
	and prove the inequality for $f^t$. Using the assumption~(\ref{eq:assumption}) and Lemma~\ref{lemma3}, the following inequalities hold
    {\small
	\begin{eqnarray}
	&&\hspace{-0.7cm}f^t \leq \sum_{i=1}^V b_i^{t-1} \left(1-\frac{1}{K}\right) + \frac{\OPT}{K}\leq \nonumber \\
	&&\hspace{-0.7cm}\leq\left[ E\left(1-\frac{1}{K}\right)^{t-1}\!\!\!\!\!\! + \frac{\OPT}{K}\sum_{u=0}^{t-2}\left(1-\frac{1}{K}\right)^u\right]\left(1-\frac{1}{K}\right)+ \frac{\OPT}{K} \nonumber\\
	&&\hspace{-0.7cm} = E\underbrace{\left(1-\frac{1}{K}\right)^{t}}_{\alpha_t^K} + \OPT\underbrace{\frac{1}{K}\sum_{u=0}^{t-1}\left(1-\frac{1}{K}\right)^u}_{\beta_t^K}
	\label{eq:alpha_beta}
	\end{eqnarray}
    }
	Since $\alpha_t^K+\beta_t^K = 1$
    \footnote{
        $\beta_t^K
        = \frac{1}{K} \sum_{u=0}^{t-1}\left(1-\frac{1}{K}\right)^u
        = (1-a) \sum_{u=0}^{t-1} a^u
        = 1-a^t
        = 1 - \left(1-\frac{1}{K}\right)^t
        = 1 - \alpha_t^K$
        for $a = \left(1-\frac{1}{K}\right)$.
    }
    and $\alpha_K = \left(1-\frac{1}{K}\right)^{K} \leq \frac{1}{e}$, the upper bound for cost function of the greedy algorithm  in $K$th iteration is 
	$f^K \leq E \frac{1}{e} + \OPT\left(1-\frac{1}{e}\right)$  \hspace{2cm} $\square$  
%

Theorem~\ref{theorem1} reveals that the approximation ratio of the greedy algorithm $\rho = \frac{f^{K}}{\OPT}$ after $K$ iterations has following upper bound 
{\small
\begin{equation}
\rho \leq \frac{\OPT(\frac{E}{\OPT}\frac{1}{e} + \left(1-\frac{1}{e}\right))}{\OPT} \leq \frac{E}{\LB(\OPT)e} + \left(1-\frac{1}{e}\right) 
\end{equation}
}
We can simply find $\LB(\OPT)$ by considering for each voxel the best $K$ rays independently. 

So far we have assumed that the greedy algorithm chooses only $K$ rays and that all rays are available in all iterations.
Since there are $L$ positions and the greedy algorithm can choose only $K$ rays at each position, some rays may be no longer available when choosing $(K+1)$th ray.
In the worst case possible, the rays from the most promising position will become unavailable.
Since we have not chosen optimal rays we can no longer achieve $\OPT$.
Nevertheless, we can still choose from rays which achieve a new optimum. 


We introduce $\overline{\OPT}_v$ as the optimum achievable after closing $v$ positions. Obviously $\overline{\OPT}_0=\OPT$. Let us assume that, when the first position is closed we cannot lose more than $R_1$, therefore $\overline{\OPT}_1=\OPT+R_1$. Without any additional assumption, $R_1$ could be arbitrarily large. 
We discuss potential assumptions later. Similarly $\overline{\OPT}_2=\OPT+R_1+R_2$, and $\overline{\OPT}_v=\OPT+\sum_{l=1}^vR_l$. The following theorem states the upper bound for $f^{LK}$ as a function of $\overline{\OPT}_v$.
%


\begin{thm}
	Upper bound $\mathit{\UB}(f^{LK})\geq f^{LK}$ of the greedy algorithm after $LK$ iterations is 
	\begin{flalign}
	 \mathit{\UB}(f^{LK}) = E \frac{1}{e} + \sum_{u=0}^{L-1}\gamma_u\overline{\OPT_u},	
	\end{flalign}
	where $\gamma_u = \left(1-\sqrt[L]{\frac{1}{e}}\right)\left(\sqrt[L]{\frac{1}{e}}\right)^{L-1-u}$
\end{thm}
\begin{proof}
	We start from the result~(\ref{eq:alpha_beta}) shown in the proof of Theorem~\ref{theorem1}. Since there is $LK$ rays achieving optimum $\overline{\OPT_0} = \OPT$, the cost function $f^K$ in $K$th iteration is bounded as follows
	\begin{nospaceflalign}
	\mkern-26muf^{K}\leq E\underbrace{\left(1-\frac{1}{LK}\right)^{K}}_{\alpha_{K}^{LK}} + \overline{\OPT_0}\underbrace{\frac{1}{LK}\sum_{u=0}^{K-1}\left(1-\frac{1}{LK}\right)^u}_{\beta_{K}^{LK}}
	\end{nospaceflalign}
	In the $(K+1)$th iteration, there are two possible cases: (i) rays from some position $l$ become not available and there is $K(L-1)$ rays available which can achieve a new optimum which is not higher than $\overline{\OPT_1}$ or (ii) all rays are available and there is still $LK$ rays which achieve $\overline{\OPT}_0=\OPT$. Noticing that the upper bound is increasing in $\overline{\OPT}_0$ and $L$, we can cover both cases by considering there is still $LK$ rays which achieves $\overline{\OPT_1}$, therefore
	\begin{eqnarray}
	&&\hspace{-1.5cm}f^{K+1}\leq (E\alpha_{K}^{LK} + \overline{\OPT_0}\beta_{K}^{LK})(1-\frac{1}{LK}) + \frac{\overline{\OPT_1}}{LK}= \nonumber \\
	&&\hspace{-1.5cm}\;\;\;\;\;\;\;\;\;=E\alpha_{K+1}^{LK} + \overline{\OPT_0}\beta_{K}^{LK}(1-\frac{1}{LK}) + \frac{\overline{\OPT_1}}{LK}
	\end{eqnarray}
	We can now continue up to the iteration $2K$ in which the upper bound is as follows
	\begin{eqnarray}
	&&\hspace{-1.5cm}f^{2K}\leq E\alpha_{2K}^{LK} + \overline{\OPT_0}\beta_{K}^{LK}\alpha_{K}^{LK} + \overline{\OPT_1}\beta_{K}^{LK} 
	\end{eqnarray}
	For $(2K+1)$th iteration the situation is similar as for ${(K+1)}$th iteration. In order to cover both cases, we consider that there is $LK$ rays which achieves $\overline{\OPT_2}$ and continue up to the $3k$th iteration, which yields the following upper bound
	\begin{eqnarray}
	f^{3K}&\leq& E\alpha_{3K}^{LK} + \overline{\OPT_0}\beta_{K}^{LK}\alpha_{2K}^{LK} +  \nonumber\\
	 && +\overline{\OPT_1}\beta_{K}^{LK}\alpha_{K}^{LK}+\overline{\OPT_2}\beta_{K}^{LK}
	\end{eqnarray}
	Finally after $LK$ iterations the upper bound is
	\begin{eqnarray}
	&&\hspace{-1.5cm}f^{LK}\leq E \alpha_{LK}^{LK} + \beta_{K}^{LK}\sum_{u=0}^{L-1}\alpha_{(L-1-u)K}^{LK}\overline{\OPT_u}\leq\nonumber\\
	&&\hspace{-.8cm} \leq E \frac{1}{e} + \sum_{u=0}^{L-1}\left(1-\sqrt[L]{\frac{1}{e}}\right)\left(\sqrt[L]{\frac{1}{e}}\right)^{L-1-u}\!\!\!\!\!\!\!\!\!\!\!\!\!\!\!\overline{\OPT_u}.
	\end{eqnarray}
	The last inequality stems from the fact that $(\alpha_{K}^{LK})^L = \alpha_{LK}^{LK} \leq \frac{1}{e}$ and that $\alpha_{K}^{LK} + \beta_{K}^{LK} = 1$.
\end{proof}

Finally we derive the upper bound of the approximation ratio $\rho = {f^{LK}}/{\OPT}$. 

\begin{thm}
	Upper bound of the approximation ratio is 
	\begin{flalign}
		\rho\geq\frac{E}{\mathsf{\LB}(\OPT)}\frac{1}{e} + \sum_{u=0}^{L-1}\gamma_u \left(1+\frac{\sum_{v=1}^uR_v}{\small{\LB(\OPT)}}\right) 
	\end{flalign}
	where $\LB(\OPT)$ is lower bound of the $\OPT$.
\end{thm}
\noindent\emph{Proof:}
	\begin{eqnarray}
	&&\hspace{-0.7cm}\rho = \frac{f^{LK}}{\OPT} \leq 	\frac{\UB(f^{LK})}{\OPT} = \frac{E \frac{1}{e} + \sum_{u=1}^{L}\gamma_u \overline{\OPT}_u}{\OPT}= \nonumber\\
	&&\hspace{-0.4cm}= \frac{\OPT(\frac{E}{\OPT}\frac{1}{e} +\sum_{u=1}^{L}\gamma_u \frac{\overline{\OPT}_u}{\OPT})}{\OPT}= \nonumber\\
	&&\hspace{-0.4cm}= \frac{E}{\OPT}\frac{1}{e} + \sum_{u=1}^{L}\gamma_u \frac{\OPT + \sum_{v=1}^uR_v}{\OPT}\leq \\
	&&\hspace{-0.4cm}\leq \frac{E}{\LB(\OPT)}\frac{1}{e} + \sum_{u=0}^{L-1}\gamma_u \left(1+\frac{\sum_{v=1}^uR_v}{\small{\LB(\OPT)}}\right)\nonumber\;\;\;\;\;\;\;\;\;\;\;\;\;\;\square
	\end{eqnarray}

The approximation ratio depends on the $\OPT$, if $\OPT=0$ then $\rho = \infty$, if $\OPT=E$ then $\rho = 1$.  If we make an assumption that each position covers only $\frac{1}{L}$ fraction of voxels, then $R_v\leq \frac{V}{L}$. Figure~\ref{fig:approximation_ratio} shows values of $\LB(\rho)$ for different ratios of $\frac{\OPT}{E}$ for this case.
\begin{figure}[t]
	\begin{center}
		\includegraphics[width=0.55\linewidth]{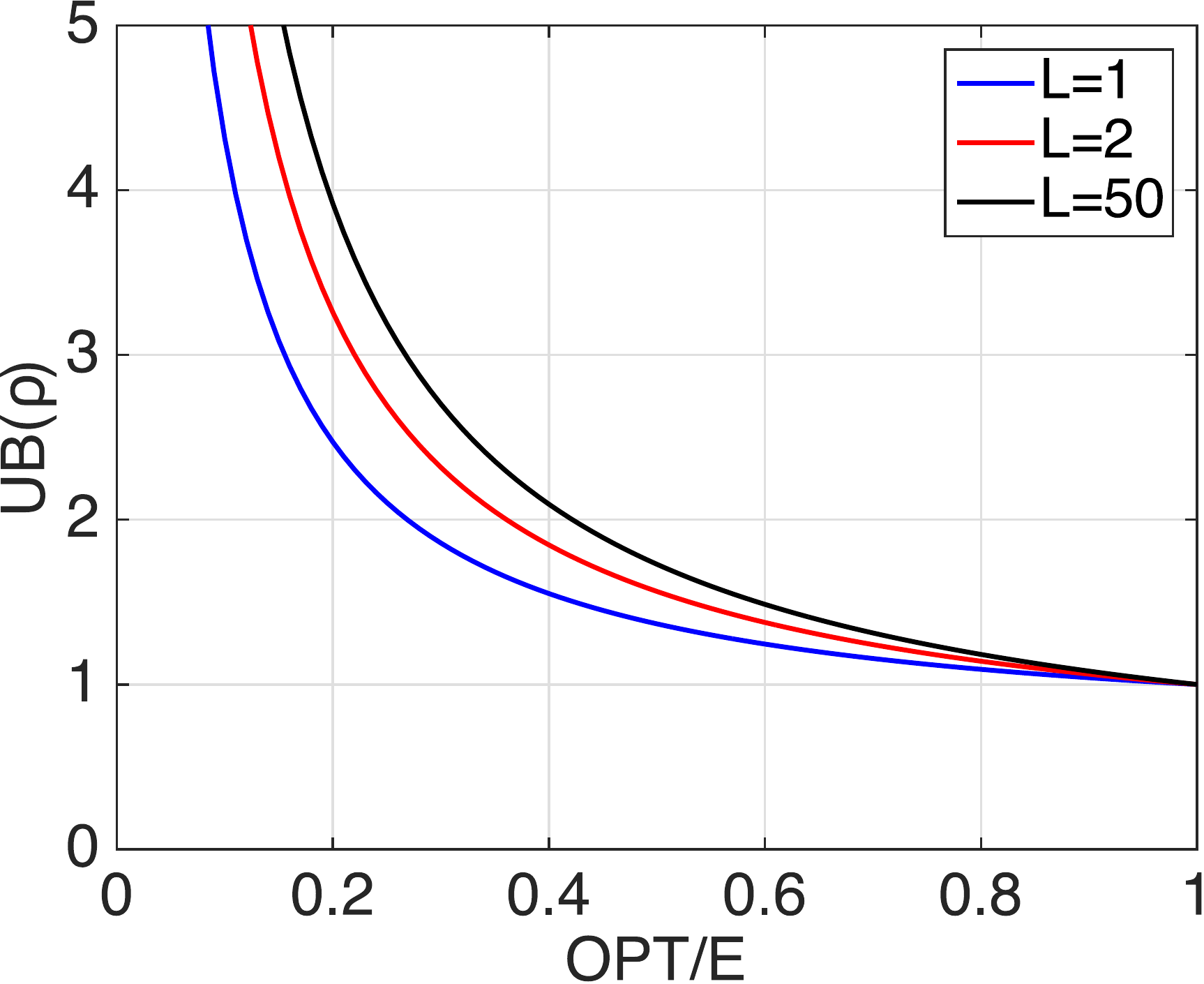}
	\end{center}
	\caption{$\text{\UB}(\rho)$ as a function of $\frac{\OPT}{E}$ ratios with $R_v\leq \frac{V}{L}$.}
	\label{fig:approximation_ratio}
\end{figure}

\subsection{Prioritized greedy planning}\label{sec:faster-greedy}

In practice we observed a significant speed up of the greedy planning (Alg.~\ref{algo:greedy0}) by imposing prioritized search for $\arg\min_j \m{b}^\mt \m{p}_{j}$.
Namely, let us denote $\Delta_j^k$ the decrease of the expected reconstruction error achieved by selecting ray $j$ in iteration $k$,
$
\Delta_j^k = \sum_i (b_i^{k-1} - b_i^{k}) = \sum_i b_i^{k-1}(1 - p_{ij}),
$
and show that it is non-increasing.
For $p_{ij}, p_{ij'} \in [0, 1]$ and $b_i^{k-1} \ge 0$ it follows that $b_i^{k-1}(1-p_{ij}) \ge b_i^{k-1}p_{ij'}(1-p_{ij})$.
Summing the inequalities for all voxels $i$, we get
{\small
\begin{equation}
\Delta_j^k = \sum_i b_i^{k-1}(1 - p_{ij}) \ge \sum_i b_i^{k-1}p_{ij'}(1 - p_{ij}) = \Delta_j^{k+1}
\end{equation}
}
for an arbitrary ray $j'$ selected in iteration $k$.
Note that $\Delta_j^k \ge \Delta_j^{k+a}$ for any $a \ge 1$.

Now, when we search for $j$ maximizing $\Delta_j^k$ in decreasing order of $\Delta_j^{k-a_j}$, $a_j \ge 1\ \forall j$, we can stop once $\Delta_j^k > \Delta_{j'}^{k-a_{j'}}$ for the next ray $j'$ because none of the remaining rays can be better than $j$.
Moreover, we can take advantage of the fact that all the remaining rays including $j$ remained sorted when updating the priority for the next iteration.
The proposed planning is detailed in Alg.~\ref{algo:greedy}.

The number of re-evaluations of $\Delta_j$ in Alg.~\ref{algo:greedy} was approximately $500\times$ smaller than in Alg.~\ref{algo:greedy0}.
Despite the sorting took about a $1/10$ of the computation time, the prioritized planning was about $30\times$ faster and took $0.3\si{s}$ on average using a single-threaded implementation.

\begin{algorithm}
    {\small
        \begin{algorithmic}[1]
            \Require \tabto{1cm} Set of rays $\s{V} = \{1, \ldots, N\}$ at positions $\s{L}$, budget $K,$ voxel costs $\m{b},$ probability vectors $\m{p}_{j}\ \forall j \in \s{V},$ mapping from ray to position $\lambda\colon \s{V} \mapsto \s{L}$
            \State $\s{J}_{l} \gets \emptyset \; \forall {l} \in \s{L}$ \TabComment{4.5cm}{No rays selected}
            \State $\Delta_j \gets \infty \; \forall {j} \in \s{V}$ \TabComment{4.5cm}{Force recompute}
            \State $S \gets (1, \dots, N)$ \TabComment{4.5cm}{Sequence of ray indices, $S(n)$ denotes the $n$th element in the sequence, $S(m{:}n)$ the subsequence from the $m$th to the $n$th element.}
            \While{$S \neq \emptyset$}
            \For{$n \in (1, \dots, |S|)$}
            \State $\Delta_{S(n)} \gets \m{b}^\mt (\m{1} - \m{p}_{S(n)})$
            \If{$n < |S| \land \Delta_{S(n)} \ge \Delta_{S(n+1)}$}
            \State \textbf{break}
            \EndIf
            \EndFor
            \State Sort subsequence $S(1:n)$ s.t. $\Delta_{S(n')} \ge \Delta_{S(n'+1)}$
            \State Merge sorted subsequences $S(1:n-1)$ and $S(n:|S|)$
            \State $j^* \gets S(1), l^* \gets \lambda(j^*)$
            \State $\s{J}_{l^*} \gets \s{J}_{l^*} \cup \{j^*\}$ \TabComment{4.5cm}{Add the best ray}
            \State $\m{b} \gets \m{b} \odot \m{p}_{j^*}$ \TabComment{4.5cm}{Update voxel costs}
            \If{$|\s{J}_{l^*}| = K$}
                \State $S \gets S \setminus \{j : \lambda(j) = l^* \}$  \TabComment{4.5cm}{Close position}
            \Else
                \State $S \gets S \setminus \{j^*\}$ \TabComment{4.5cm}{Remove $j^*$ from $S$}
            \EndIf
            \EndWhile
            \State \Return Selected rays $\s{J}_l$ at every position $l \in \s{L}$
        \end{algorithmic}
    }
    \caption{Prioritized greedy planning}\label{algo:greedy}
\end{algorithm}

\section{Experiments}

\paragraph{Dataset}
All experiments were conducted on selected sequences from categories \emph{City} and \emph{Residential} from the KITTI dataset~\cite{Geiger-2013-IJRR}.
We first brought the point clouds (captured by the Velodyne HDL-64E laser scanner) to a common reference frame using the localization data from the inertial navigation system (OXTS RT 3003 GPS/IMU)
and created the ground-truth voxel maps from these.
The voxels traced from the sensor origin towards each measured point were updated as empty except for the voxels incident with any of the end points which were updated as occupied for each incident end point.
The dynamic objects were mostly removed in the process since the voxels belonging to these objects were also many times updated as empty while moving.
All maps used axis-aligned voxels of edge size $\SI{0.2}{m}$.

For generating the sparse measurements, we consider an SSL sensor with the field of view of $120^{\circ}$ horizontally and $90^{\circ}$ vertically discretized in $160\times120=19200$ directions.
At each position, we select $K=200$ rays and ray-trace in these directions until an occupied voxel is hit or the maximum distance of $48\si{m}$ is reached.
Only the rays which end up hitting an occupied voxel produce valid measurements, as is the case with the time-of-flight sensors.
Local maps $\m{x}_l$ and $\m{y}_l$ contain volume of $64\si{m} \times 64\si{m} \times 6.4\si{m}$ discretized into $320\times320\times32$ voxels.


\subsection{Active 3D mapping}

In this experiment, we used $17$ and $3$ sequences from the \emph{Residential} category for training and validation, respectively, and $13$ sequences from the \emph{City} category for testing.
We evaluate the iterative planning-learning procedure described in Sec.~\ref{sec:mapping}.
For learning the mapping networks, we used learning rate $\alpha = 10^{-3} (1/8)^{\lceil i / 10 \rceil}$ based on epoch number $i$, batch size $1$, and momentum $0.99$.
Networks $\m{h}_{\theta^0}, \ldots, \m{h}_{\theta^3}$ were trained for $20$ epochs.

The ROC curves shown in Fig.~\ref{fig:roc} (left) are computed using ground-truth maps $\m{Y}$ and predicted global occupancy maps $\hat{\m{Y}}$.
The performance of the $\m{h}_{\theta^3}$ network (denoted \emph{Coupled}) significantly outperforms the $\m{h}_{\theta^3}$ network (\emph{Random}), which shows the benefit of the proposed iterative planning-mapping procedure.
Examples of reconstructed global occupancy maps are shown in Fig.~\ref{fig:inouts}. 
Note that the valid measurements covered around $3\%$ of the input voxels.


\begin{figure}
    \begin{center}
  	\includegraphics[width=0.495\linewidth]{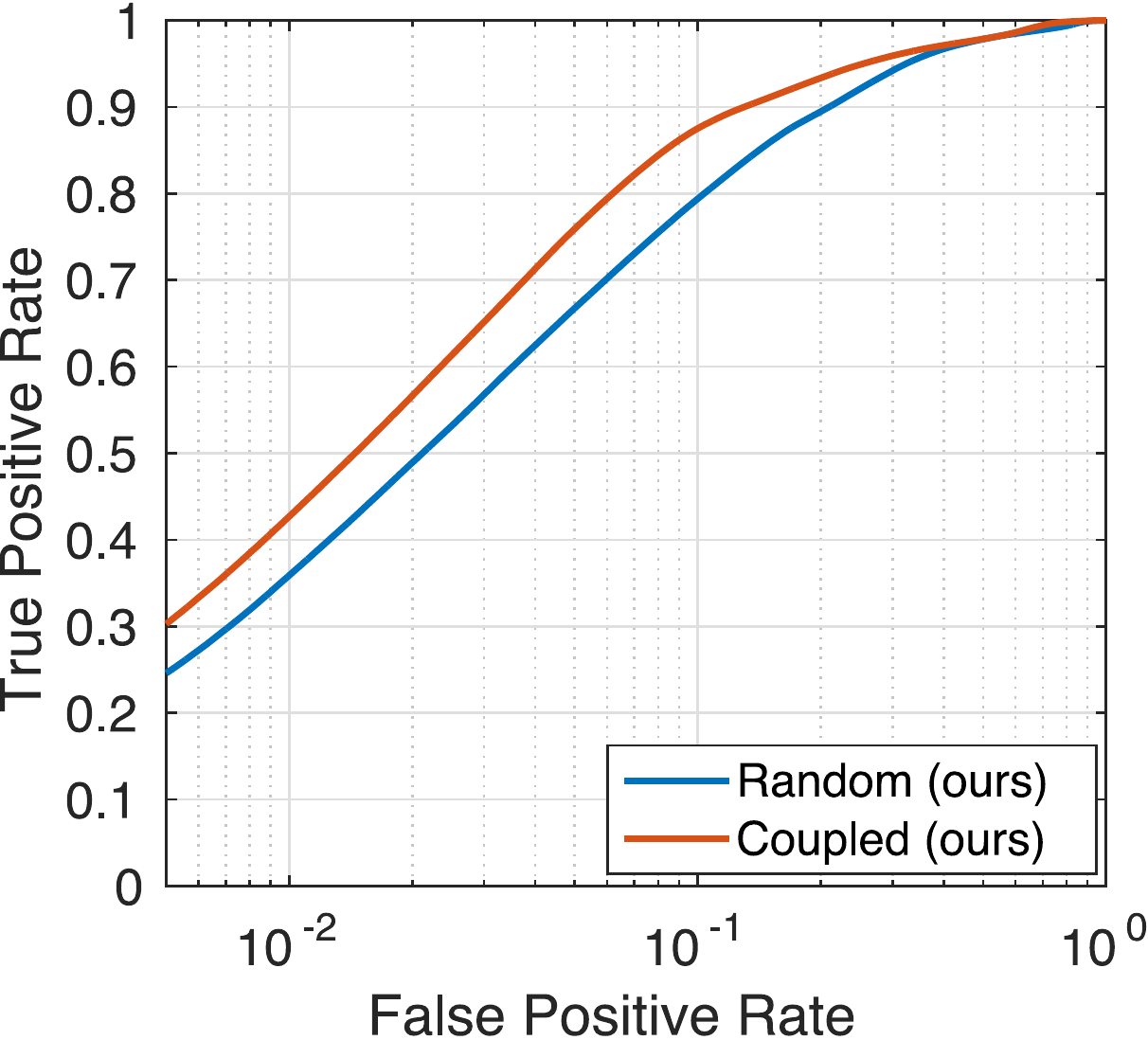}    
    \includegraphics[width=0.495\linewidth]{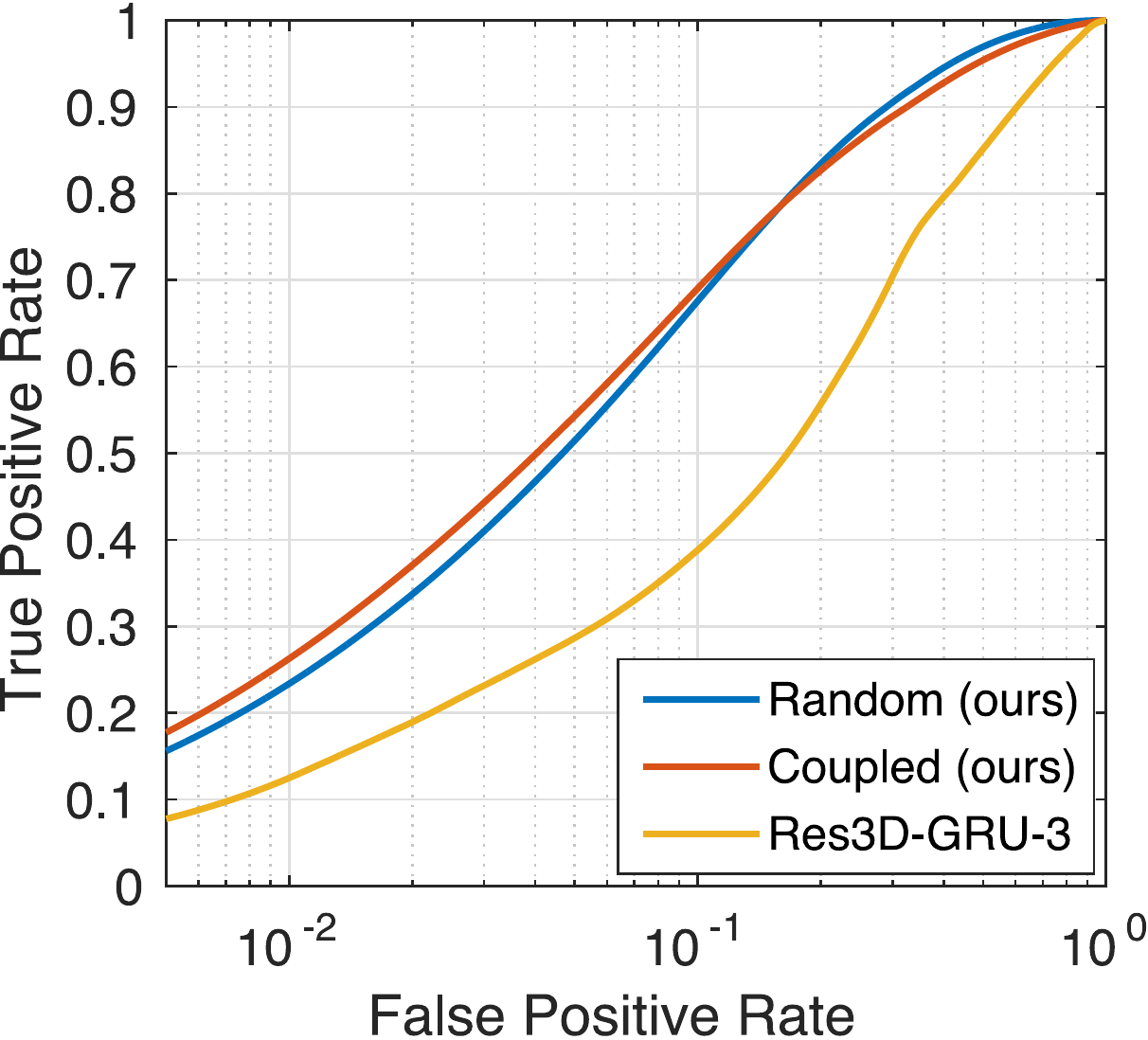}
    \end{center}
    \caption{ROC curves of occupancy prediction from active 3D mapping on test sets.
    \textbf{Left:}
    \emph{Random} denotes the global occupancy $\hat{\m{Y}}$ obtained by using $\m{h}_{\theta^0}$ with random sparse measurements, \emph{Coupled} the occupancy obtained by using $\m{h}_{\theta^3}$ with the prioritized greedy planning.
    The voxels which are more than $1\si{m}$ from what could possibly be measured are excluded, together with the false positives which can be attributed to discretization error (in 1-voxel distance from an occupied voxel).
    \textbf{Right:} \emph{Random} denotes the local occupancy maps $\hat{\m{y}}_l$ obtained by using $\m{h}_{\theta^0}$, \emph{Coupled} the maps obtained by using $\m{h}_{\theta^1}$, and \emph{Res3D-GRU-3} denotes the reconstruction obtained by the network adapted from \cite{Choy-2016-ECCV}.
    }
    \label{fig:roc}
	\label{fig:choy}
\end{figure}

\begin{figure}
    \begin{center}   
        \includegraphics[width=0.99\linewidth]{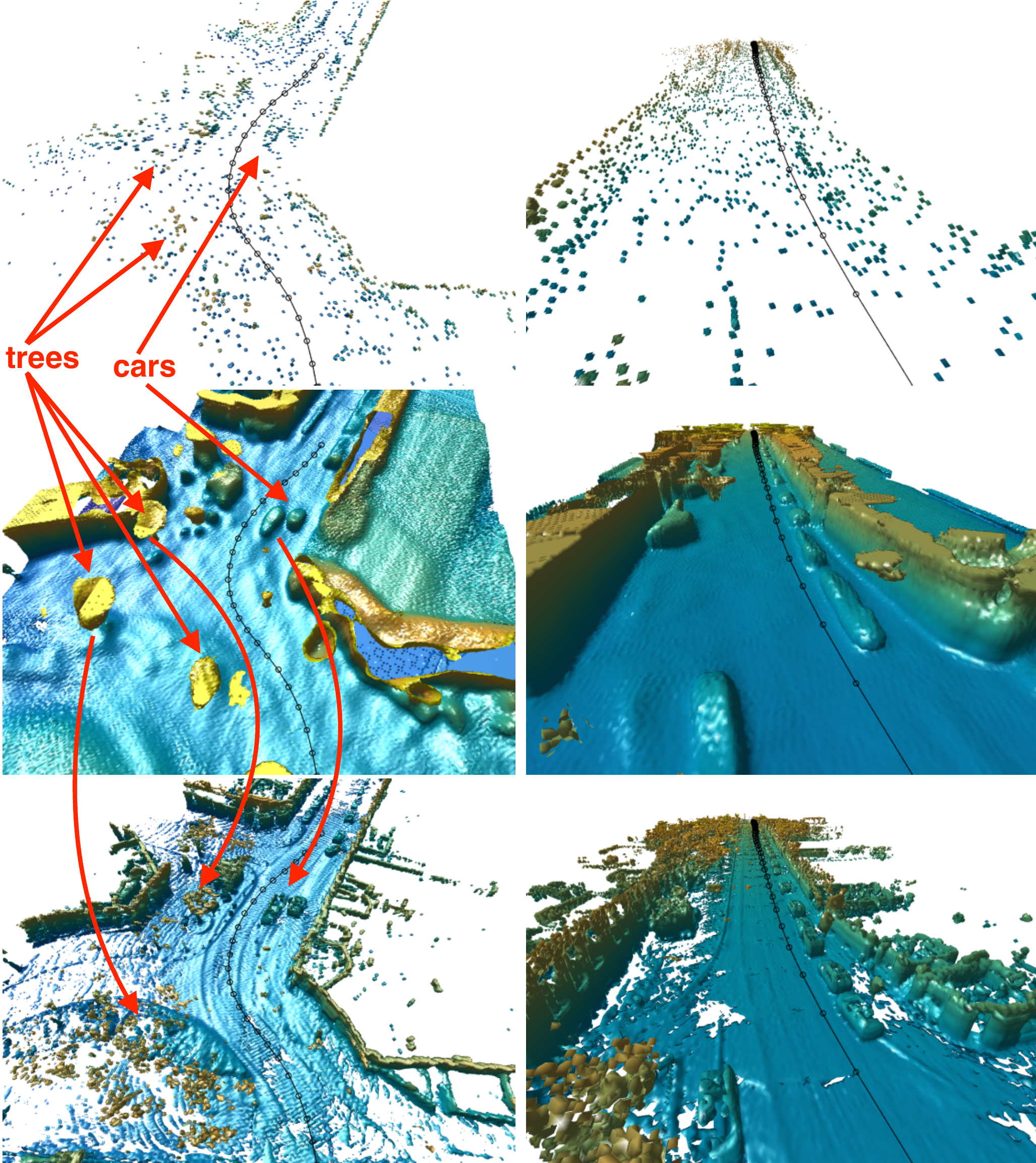}
    \end{center}
    \caption{Examples of global map reconstruction.
    \textbf{Top:} Sparse measurement maps $\m{X}$.
    \textbf{Middle:} Reconstructed occupancy maps $\hat{\m{Y}}$ in form of isosurface.
    \textbf{Bottom:} Ground-truth maps $\m{Y}$.
    The black line denotes trajectory of the car.
    }
    \label{fig:inouts}
\end{figure}

\subsection{Comparison to a recurrent image-based architecture}\label{sec:choy}
We provide a comparison with the image-based reconstruction method of Choy \etal~\cite{Choy-2016-ECCV}.
Namely, we modify
their residual \emph{Res3D-GRU-3} network
to use sparse depth maps of size $160 \times 120$ instead of RGB images.
The sensor pose corresponding to the last received depth map was used for reconstruction.
The number of views were fixed to $5$, with $K = 200$ randomly selected depth-measuring rays in each image.
For this experiment, we used 20 sequences from the \emph{Residential} category---18 for training, 1 for validation and 1 for testing.
Since the \emph{Res3D-GRU-3} architecture is not suited for high-dimensional outputs due to its high memory requirements, we limit the batch size to $1$ and the size of the maps to $128 \times 128 \times 32$, which corresponds to $16 \times 16 \times 4$ recurrent units.
Our mapping network was trained and tested on voxel maps instead of depth images.

The corresponding ROC curves, computed from local maps $\m{y}_l$ and $\hat{\m{y}}_l$, are shown in Fig.~\ref{fig:choy} (right).
Both $\m{h}_{\theta^0}$ and $\m{h}_{\theta^1}$ networks outperforms the \emph{Res3D-GRU-3} network.
We attribute this result mostly to the fact that our method is implicitly provided the known trajectory, while the \emph{Res3D-GRU-3} network is not. 
Another reason may be the ray-voxel mapping which is also known implicitly in our case, compared to \cite{Choy-2016-ECCV}.

\section{Conclusions}
We have proposed a computationally tractable approach for the very high-dimensional active perception task.
The proposed 3D-reconstruction CNN outperforms a state-of-the-art approach by $20\%$ in recall, and it is shown that when learning is coupled with planning, recall increases by additional $8\%$ on the same false positive rate.
The proposed prioritized greedy planning algorithm seems to be a promising direction with respect to on-board reactive control since it is about $30\times$ faster and requires only $1/500$ of ray evaluations compared to a na{\" i}ve greedy solution.
\vspace{-0.5em}
\section*{Acknowledgment}
\vspace{-0.3em}
The research leading to these results has received funding from the European Union under grant agreement FP7-ICT-609763 TRADR and No.~692455 Enable-S3, from the Czech Science Foundation under Project~17-08842S, and from the Grant Agency of the CTU in Prague under Project SGS16/161/OHK3/2T/13.
\newpage

{\small
\bibliographystyle{ieee}
\bibliography{paper}
}

\end{document}